\documentclass[runningheads]{llncs}
\usepackage{geometry}                
\usepackage{graphicx}
\usepackage{amssymb}
\usepackage{epstopdf}
\usepackage[noend]{algorithm2e}
\usepackage{tabularx}
\usepackage[disable]{todonotes}
\usepackage{multirow}
\usepackage{relsize}
\usepackage{graphicx}

\usepackage{subfig}
\usepackage{mathtools}

\DeclarePairedDelimiter\abs{\lvert}{\rvert}

\usepackage{ amssymb }
\usepackage{tikz}
\usetikzlibrary{arrows}
\usepackage{varwidth}
\usepackage{listings}
\SetCommentSty{mycommfont}

\newtheorem{observation}{Observation}

\newcommand{\targ}{\prod c^*_i}
\newcommand{\dist}{\mathcal{D}}

\newcommand{\genQ}{q}
\newcommand{\oneposQ}{\emph{1}Pos}
\newcommand{\posQ}{Pos}
\newcommand{\memQ}{\emph{Mem}}
\newcommand{\subQ}{\emph{Sub}}
\newcommand{\supQ}{\emph{Sup}}
\newcommand{\eqQ}{\emph{EQ}}
\newcommand{\pacQ}{\emph{EX}}

\newcommand{\supQi}[1]{\emph{Sup}}

\newcommand{\genC}{\#q}
\newcommand{\oneposC}{\#$1$Pos}
\newcommand{\posC}{\#Pos}
\newcommand{\memC}{\#\emph{Mem}}
\newcommand{\subC}{\#\emph{Sub}}
\newcommand{\supC}{\#\emph{Sup}}
\newcommand{\eqC}{\#\emph{EQ}}
\newcommand{\pacC}{\#\emph{EX}}

\newcommand{\genCi}[1]{\#q_{#1}}

\newcommand{\memCi}[1]{\#\emph{Mem}_{#1}}
\newcommand{\subCi}[1]{\#\emph{Sub}_{#1}}
\newcommand{\supCi}[1]{\#\emph{Sup}}
\newcommand{\eqCi}[1]{\#\emph{EQ}_{#1}}

\newcommand{\ntreef}{\mathfrak{c}}
\newcommand{\ntree}{\mathfrak{c}_{sub}}
\newcommand{\eqhard}{\mathfrak{C}}

\newcommand{\VC}{\mathcal{V}\mathcal{C}}

\newcommand{\Boxtimes}{\mathlarger{\mathlarger{\boxtimes}}}

\makeatletter
\def\moverlay{\mathpalette\mov@rlay}
\def\mov@rlay#1#2{\leavevmode\vtop{%
   \baselineskip\z@skip \lineskiplimit-\maxdimen
   \ialign{\hfil$\m@th#1##$\hfil\cr#2\crcr}}}
\newcommand{\charfusion}[3][\mathord]{
    #1{\ifx#1\mathop\vphantom{#2}\fi
        \mathpalette\mov@rlay{#2\cr#3}
      }
    \ifx#1\mathop\expandafter\displaylimits\fi}
\makeatother

\newcommand{\cupdot}{\charfusion[\mathbin]{\cup}{\cdot}}
\newcommand{\bigcupdot}{\charfusion[\mathop]{\bigcup}{\cdot}}

\newcommand{\disClass}{C_{\cupdot}}

\begin{document}
\raggedbottom

\title{Modularity in Query-Based Concept Learning} 
%
%
\author{Benjamin Caulfield\and
Sanjit A. Seshia
}
\authorrunning{B. Caulfield \& S. Seshia}
%
\institute{University of California, Berkeley CA 94720, USA 
\email{\{bcaulfield,sseshia\}@berkeley.edu}\\
}
\maketitle              
\begin{abstract}
We define and study the problem of modular concept learning, that is, learning a concept that is a cross product of component concepts.
If an element's membership in a concept depends solely on it's membership in the components, learning the concept as a whole can be reduced to learning the components. 
We analyze this problem with respect to different types of oracle interfaces, defining different sets of queries.
If a given oracle interface cannot answer questions about the components, learning can be difficult, even when the components are easy to learn with the same type of oracle queries.
While learning from superset queries is easy, learning from membership, equivalence, or subset queries is harder. 
However, we show that these problems become tractable when oracles are given a positive example and are allowed to ask membership queries.
\end{abstract}

%
%
%

\section{Introduction}

Inductive synthesis or inductive learning 
is the synthesis of programs (concepts) from examples or other observations. 
Inductive synthesis has found application in formal methods, program analysis,
software engineering, and related areas, for problems such as 
invariant generation (e.g.~\cite{garg2014ice}),
program synthesis (e.g.,~\cite{solar2006combinatorial}),
and compositional reasoning (e.g.~\cite{cobleigh2003learning}).
Most inductive synthesis follows the query-based learning model, where the
learner is allowed to make queries about the target concept to an oracle. 
Using the correct set of oracles can result in the polynomial time learnability of otherwise unlearnable sets \cite{angluin1988queries}. 
Using queries for software analysis is becoming increasingly popular (e.g., \cite{vaandrager17,howar2018active}).
The special nature of query-based learning for formal synthesis, where a program is automatically generated to fit a high-level specification through interaction with oracles,
has also been formalized \cite{jha2017theory}. 


In spite of this progress,
most algorithms for inductive learning/synthesis are monolithic; that is, even
if the concept (program) is made up of components, the algorithms seek to
learn the entire concept from interaction with an oracle.
In contrast, in this paper, we study the setting of {\em modular concept learning},
where a learning problem is analyzed by breaking it into independent components. 
If an element's membership in a concept depends solely on its membership in the 
components that make up the concept, 
learning the concept as a whole can be reduced to learning the components. 
We study concepts that are the Cartesian products (i.e., cross-products) of their 
component concepts. Such concept arise in several applications: 
(i) in invariant generation, an invariant that is the conjunction of other component
invariants;
(ii) in compositional reasoning, an automaton that is the product of individual automata encapsulating different aspects of an environment model, and
(iii) in program synthesis, a product program whose state space is the product of the state spaces of individual component programs.
Modular concept learning can improve the {\em efficiency} of learning since the complexity of several
query-based learning algorithms depends on the size of the concept (e.g. automaton) to
be learned, and, as is well known, this can grow exponentially with the number of components.
Besides improving efficiency of learning, from a software engineering perspective,
modular concept learning also has the
advantage of {\em reuse} of component learning algorithms.

We will focus on the oracle queries given in Table \ref{table:queries} and show several results, including both upper and lower bounds. 
We show learning cross-products from superset queries is no more difficult than learning each individual concept. 
Learning cross-products from equivalence queries or subset queries is intractable, while learning from just membership queries is polynomial, though somewhat expensive. 
We show that when a learning algorithm is allowed to make membership queries and is give a single positive example, previously intractable problems become tractable. 
We show that learning the disjoint unions of sets is easy.
Finally, we discuss the computational complexity of PAC-learning and show how it can be improved when membership queries are allowed. 

\begin{table}
\begin{center}
  \begin{tabularx}{\textwidth}{| c | c | c | X | }
    \hline
    Query Name & Symbol & Complexity & Oracle Definition \\ \hline
    Single Positive Query & $\oneposQ$ & $\oneposC(c)$ & Return a fixed $x \in c^*$ \\ \hline
    Positive Query & $\posQ$ & $\posC(c)$ & Return an $x\in c^*$ that has not yet been given as a positive example (if one exists)\\ \hline
    Membership Query & $\memQ$ & $\memC(c)$ & Given string $s$, return true iff $s \in c^*$ \\ \hline
    Equivalence Query & $\eqQ$ & $\eqC(c)$ & Given $c \in C$, return true if $c=c^*$ otherwise return $x \in (c \backslash c^*) \cup (c^* \backslash c)$\\ \hline 
    Subset Query & $\subQ$ & $\subC(c)$ & Given $c \in C$, return `true' if $c \subseteq c^*$ \mbox{  } otherwise return some $x \in c \backslash c^*$ \\ \hline
    Superset Query & $\supQ$ & $\supC(c)$ & Given $c \in C$, return `true' if $c \supseteq c^*$  otherwise return some $x \in c^* \backslash c$\\ \hline
    Example Query & $\pacQ_\dist$ & $\pacC(c, \dist)$ & Samples $x$ from $\dist$ and returns $x$ with a label indicating whether $x \in c^*$. \\ \hline
  \end{tabularx}
\end{center}
\caption{Types of queries studied in this paper.}
\label{table:queries}
\end{table}

\subsection{A Motivating Example}

    %

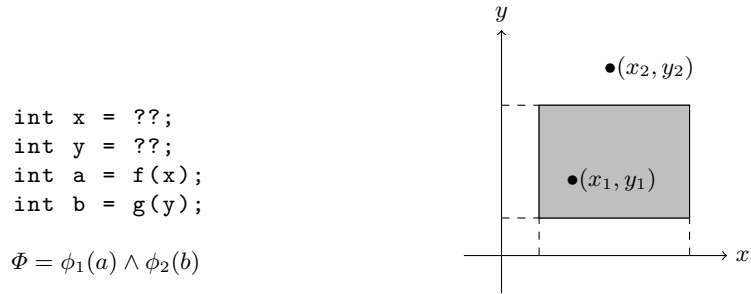
\begin{figure}[h]
\centering
\begin{minipage}[b]{0.4 \textwidth}
\begin{lstlisting}
    int x = ??;
    int y = ??;
    int a = f(x);
    int b = g(y);
    
    $\Phi = \phi_1(a) \wedge \phi_2(b)$
\end{lstlisting}
\end{minipage}
\qquad
\begin{tikzpicture}
     \coordinate (BL) at (0.5,0.5);
     \coordinate (TR) at (2.5,2);
      \draw[->] (-0.5,0) -- (3,0) node[right] {$x$};
      \draw[->] (0,-0.5) -- (0,3) node[above] {$y$};
      \draw [draw=black] (BL) rectangle (TR);
      \filldraw [fill=lightgray, draw=black] (BL) rectangle (TR);
      \draw[dashed] (0,0.5) -- (BL);
      \draw[dashed] (0,2) -- (0.5,2);
      \draw[dashed] (0.5,0) -- (BL);
      \draw[dashed] (2.5,0) -- (2.5,0.5);
      \draw(.75,1) node[anchor = west]{\textbullet $(x_1,y_1)$};
      \draw(1.25,2.5) node[anchor = west]{\textbullet $(x_2,y_2)$};
\end{tikzpicture}
\caption{A simple partial program to be synthesized to satisfy a specification $\Phi$ (left) and the correct set of initial values for $x$ and $y$ (right).}
\label{sketch}
\label{sketchsolutions}
\end{figure}

To illustrate the learning problem, consider the sketching problem given in Figure \ref{sketch}.
Say we wanted to find the set of possible initial values for $x$ and $y$ that can replace the $??$ values so that the program satisfies $\Phi$.

Looking at the structure of this program and specification, we can see that the correctness of these two variables are independent of each other. 
Correct $x$ values are correct independent of $y$ and vice-versa.
Therefore, the set of settings will be the cross product of the acceptable settings for each variable.  
If an oracle can answer queries about correct $x$ or $y$ values separately, then the oracle can simply learn the acceptable values separately and take their Cartesian product. 

      %

If the correct values form intervals, the correct settings will look something like the rectangle shown in Figure \ref{sketchsolutions}.
An algorithm for learning this rectangle can try to simulate learning algorithms for each interval by acting as the oracle for each sublearner. 
For example, if both sublearners need a positive example, the learner can query the oracle for a positive example. 
Given the example $(x_1,y_1)$ as shown in the figure, the learner can then pass $x_1$ and $y_1$ to the sublearners as positive examples. 
However, this does not apply to negative examples, such as $(x_2, y_2)$ in the figure. 
In this example, $x_2$ is in its target interval, but $y_2$ is not. 
The learner has no way of knowing which subconcept a negative element fails on.
Handling negative counterexamples is one of the main challenges of this paper.

\section{Notation}
In the following proofs, we assume we are given concept classes $C_1, C_2, \dots, C_k$ defined over sets $X_1$, $X_2$, \dots, $X_k$. 
Each $c^*_i$ in each $C_i$ is learnable from algorithm $A_i$ (called \emph{sublearners}) using queries to an oracle that can answer any queries in a set $Q$. 
This set $Q$ contains the available types of queries and is a subset of the queries shown in Table \ref{table:queries}.
For example, if $Q := \{ \memQ, \eqQ \}$, then each $A_i$ can make membership and equivalence queries to its corresponding oracle. 
 \todo{Q is almost always a singleton. Should we just call it a query instead of a set?}

For each query $q \in Q$, we say algorithm $A_i$ makes $\genCi{i}$ (or $\genCi{i}(c^*_i)$) many $\genQ$ queries to the oracle in order to learn concept $c^*_i$, dropping the index $i$ when necessary . 
We replace the term $\genC$ with a more specific term when the type of query is specified.
For example, an algorithm $A$ might make $\memC$ many membership queries to learn $c$. 
\todo{What background information, i.e., explanation of queries, should I include, if any?}

Unless otherwise stated, we will assume any index $i$ or $j$ ranges over the set $\{ 1 \dots k \}$.
We write $\prod S_i$ or $S_1 \times \dots \times S_k$ to refer to the $k$-ary Cartesian product (i.e., cross-product) of sets $S_i$. 
We use $S^k$ to refer to $\prod_{i=1}^k S$. 

We use vector notation $\vec{x}$ to refer to a vector of elements $(x_1,\dots, x_k)$, $\vec{x}[i]$ to refer to $x_i$, and $\vec{x}[i \leftarrow x'_i]$ to refer to $\vec{x}$ with $x'_i$ replacing value $x_i$ at position $i$. 
We define $\Boxtimes^k_{i=1} C_i := \{ \prod c_i \mid c_i \in C_i, i \in \{1,\dots,k\} \}$. 
We write $\vec{c}$ or $\prod c_i$ for any element of $\Boxtimes^k_{i=1} C_i $ and will often denote $\vec{c}$ by $(c_1, \dots, c_k)$ in place of $\prod c_i$. 
The target concept will be represented as $c^*$ or $\targ$ which equals $(c^*_1, \dots, c^*_k)$.

\todo{mention representations}

The results below answer the following question:
\begin{quote}
For different sets of queries $Q$, what is the bound on the number of queries to learn a concept in $\Boxtimes C_i $ as a function of each $\genCi{i}$ for each $q \in Q$?
\end{quote}

The proofs in this paper make use of the following simple observation:
\begin{observation}
\label{subobs}
For sets $S_1, S_2, \dots, S_k$ and $T_1, T_2, \dots, T_k$, assume $\prod S_i \ne \emptyset$.  Then $\prod S_i \subseteq \prod T_i$ if and only if $S_i \subseteq T_i$, for all $i$.
\end{observation}

\section{Simple Lower Bounds}
This section introduces some fairly simple lower bounds.
We will start with a lower-bound on learnability from positive examples. 

\begin{proposition}
There exist concepts $C_1$ and $C_2$ that are each learnable from constantly many positive queries, such that $C_1 \times C_2$ is not learnable from any number of positive queries. 
\end{proposition}
\begin{proof}
Let $C_1 := \{ \{a\}, \{a,b\} \}$ and set $C_2 := \{ \mathbb{N}, \mathbb{Z} \backslash \mathbb{N} \}$. \todo{double check $:=$ is used for defining everywhere}
To learn the set in $C_1$, pose two positive queries to the oracle, and return $\{a,b\}$ if and only if both $a$ and $b$ are given as positive examples. 
To learn $C_2$, pose one positive query to the oracle and return $\mathbb{N}$ if and only if the positive example is in $\mathbb{N}$. 
An adversarial oracle for $C_1 \times C_2$ could give positive examples only in the set $\{a\} \times \mathbb{N}$. 
Each new example is technically distinct from previous examples, but there is no way to distinguish between the sets $\{a\}\times \mathbb{N}$ and $\{a,b\} \times \mathbb{N}$ from these examples. 
\end{proof}

Now we will show lower bounds on learnability from $\eqQ$, $\subQ$, and $\memQ$. 
We will see later that this lower bound is tight when learning from membership queries, but not equivalence or subset queries.

\begin{proposition}
There exists a concept $C$ that is learnable from $\genC$ many queries posed to $Q \subseteq \{ \memQ, \eqQ, \subQ \}$ such that learning $C^k$ requires $(\genC)^k$ many queries.   \todo{Should I explicitly handle infinite and finite cases separately? Should I include bigO notation on the infinite case?}
\end{proposition}
\begin{proof}
Let $C := \{ \{j\} \mid j \in \{0 \dots m\} \}$. 

We can learn $C$ in $m$ membership, subset, or equivalence queries by querying $j \in c^*$, $\{ j \} \subseteq c^*$, or $\{j\} = c^*$, respectively. 

However, a learning algorithm for $C^k$ requires more than $m^k$ queries. 
To see this, note that  $C^k$ contains all singletons in a space of size $(m+1)^k$. 

So for each subset query $\{x\} \subseteq c^*$, if $\{j\} \ne c^*$, the oracle will return $j$ as a counterexample, giving no new information.  
Likewise, for each equivalence query $\{j\} = c^*$, if $\{j\} \ne c^*$, the oracle can return $j$ as a counterexample.
Therefore, any learning algorithm must query $x \in c^*$, $\{ x \} \subseteq c^*$, or $\{x\} = c^*$ for $(m+1)^k - 1$ values of $x$.
\end{proof}

\section{Learning From Superset Queries}This section introduces arguably the simplest positive result of the paper: when using superset queries, learning cross-products of concepts is as easy as learning the individual concepts. 

Like all positive results in this paper, this is accomplished by algorithm that takes an oracle for the cross-product concept $\prod c_i^*$ and simulates the learning process for each sublearner $A_i$ by acting as an oracle for each such sublearner. 

\todo{make sure subconcept sublearner, etc, is defined}

\todo{should I give more intuition? It's all in the proof fwiw}




\begin{proposition}

If $Q = \{ \supQ \}$, then there is an algorithm that learns any concept $\targ \in \prod C_i $ in $\sum \supCi{i}(c^*_i)$ queries.  
\end{proposition}
\begin{proof}
Algorithm \ref{supalg} learns $\Boxtimes C_i $ by simulating the learning of each $A_i$ on its respective class $C_i$. 
The algorithm asks each $A_i$ for superset queries $S_i \supseteq c_i^*$, queries the product $\prod S_i$ to the oracle, and then uses the answer to answer at least one query to some $A_i$. \todo{should there be a special symbol for queries instead of just $\supseteq$}
Since at least one $A_i$ receives an answer for each oracle query, at most $\sum \supCi{i}(c^*_i)$ queries must be made in total.

We will now show that each oracle query results in at least one answer to an $A_i$ query (and that the answer is correct). 
The oracle first checks if the target concept is empty and stops if so.
If no concept class contains the empty concept, this check can be skipped. 
At each step, the algorithm poses query $\prod S_i$ to the oracle. 
If the oracle returns 'yes' (meaning $\prod S_i \supseteq \targ$), then  $S_i \supseteq c_i^*$ for each $i$ by Observation \ref{subobs}, so the oracle answers 'yes' to each $A_i$. 
If the oracle returns 'no', it will give a counterexample $\vec{x} = (x_1,\dots,x_k) \in \targ \backslash \prod S_i$. 
There must be at least one $x_i \not\in S_i$ (otherwise, $\vec{x}$ would be in $\prod S_i$). 
So the algorithm checks $x_j \in S_j$ for all $x_j$ until an $x_i \not\in S_i$ is found. 
Since $\vec{x} \in \targ$, we know $x_i \in c_i^*$, so $x_i \in c_i^* \backslash S_i$, so the oracle can pass $x_i$ as a counterexample to $A_i$. 

Note that once $A_i$ has output a correct hypothesis $c_i$, $S_i$ will always equal $c_i$, so counterexamples must be taken from some $j \ne i$. 
\end{proof}

\begin{algorithm}[H]
\label{supalg}
\SetAlgoLined
\KwResult{Learn $\prod C_i $ from Superset Queries}
\If{$\emptyset \in C_i$ for some $i$}{
	Query $\emptyset \supseteq \targ$\;
	\If{$\emptyset \supseteq \targ$}{
		\Return{$\emptyset$}
	}
}
\For{$i = 1 \dots k$}{
	Set $S_i$ to initial subset query from $A_i$
}
 \While{Some $A_i$ has not completed}{
  Query $\prod S_i$ to oracle\;
  \eIf{$\prod S_i \supseteq c^*$ }{
   Answer $S_i \supseteq c_i^*$ to each $A_i$\;
   Update each $S_i$ to new query\;
   }{
 	Get counterexample $\vec{x} = (x_1,\dots,x_k)$
   	\For{i = 1 \dots k}{
   		\If{$x_i \not\in S_i$}{
			Pass counterexample $x_i$ to $A_i$\;
			Update $S_i$ to new query\;
			} 
  		}
 	}
	\For{i = 1 \dots k}{
		\If{$A_i$ outputs $c_i$}{
			Set $S_i := c_i$\;
		}
	}
 }
 \Return{$\prod c_i$}\;
 \caption{Algorithm for learning from Subset Queries}
\end{algorithm}

\section{Learning From Membership Queries and One Positive Example}
Ideally, learning the cross-product of concepts should be about as easy as learning all the individual concepts.
The last section showed this is not the case when learning with equivalence, subset, or membership queries.
However, when the learner is given a single positive example and allowed to make membership queries, the number of queries becomes tractable. 
This is due to the following simple observation.

\begin{observation}
\label{posobs}
Fix sets $S_1, S_2, \dots, S_k$, points $x_1, x_2, \dots, x_k$ and an index $i$. 
If $x_j \in S_j$ for all $j \ne i$, then $(x_1, x_2, \dots, x_k) \in \prod S_i$ if and only if $x_i \in S_i$.
\end{observation}

So, given a positive example $\vec{p}$,  we can see that $\vec{p}[j \leftarrow x_j] \in \targ$ if and only if $x_j \in c^*_j$.
This fact is used to learn using subset or equivalence queries with the addition of membership queries and a positive example.
The algorithm is fairly similar for equivalence and subset queries, and is shown as a single algorithm in Algorithm \ref{lineqalg}.

\begin{proposition}
If $Q \in \{\{\subQ\}, \{\eqQ\}\}$ and a single positive example $\vec{p} \in \targ$ is given, then $\targ$ is learnable in $\sum \genCi{i}$ queries from $Q$ (i.e., subset or equivalence queries) and $k \cdot \sum \genCi{i}$ membership queries. 
\end{proposition}
\begin{proof}
The learning process for either subset or equivalence queries is described in Algorithm \ref{lineqalg}, with differences marked in comments. 
In either case, once the correct $c_j$ is found for any $j$, $S_j$ will equal $c_j$ for all future queries, so any counterexamples must fail on an $i \ne j$. 

We separately show for each type of query that a correct answer is given to at least one learner $A_i$ for each subset (resp. equivalence) query to the cross-product oracle. 
Moreover, at most $k$ membership queries are made per subset (resp. equivalence) query, yielding the desired bound. 

Subset Queries:
For each subset query $\prod S_i \subseteq \targ$, the algorithm either returns `yes' or gives a counterexample $\vec{x} = (x_1, \dots, x_k) \in \prod S_i \backslash \targ$. 
If the algorithm returns 'yes', then by Observation \ref{subobs} $S_i \subseteq c^*_i$ for all $i$, so the algorithm can return 'yes' to each $A_i$. 
Otherwise, $\vec{x} \not\in \targ$, so there is an $i$ such that $x_i \not\in c^*_i$. 
By Observation \ref{posobs} the algorithm can query $\vec{p}[j \leftarrow x_j]$ for all $j$ until the $x_i \not\in c^*_i$ is found. 

Equivalence Queries:
For each equivalence query $\prod S_i = \targ$, the algorithm either returns 'yes', or gives a counterexample $\vec{x} = (x_1, \dots, x_k)$.
If the algorithm returns `yes', then a valid target concept is learned. 
Otherwise, either $\vec{x} \in \prod S_i \backslash \targ$ or $\vec{x} \in \targ \backslash \prod S_i$.
In the first case, as with subset queries,  the algorithm uses $k$ membership queries to query $\vec{p}[j \leftarrow x_j]$ for all $j$.
Once the $x_i \not\in c^*_i$ is found it is given to $A_i$ as a counterexample.
In the second case, as with superset queries, the algorithm checks if $x_j \in S_j$ for all $j$ until the $x_i \not\in c^*_i$ is found and given to $A_i$.
\end{proof}

\begin{algorithm}[H]
\label{lineqalg}
\SetAlgoLined
\KwResult{Learn $\targ \in \Boxtimes C_i $}
\For{$i = 1 \dots k$}{
	Set $S_i$ to initial query from $A_i$
}
 \While{Some $A_i$ has not completed}{
        Query $\prod S_i$ to oracle\;
        \eIf{The Oracle returns `yes'}{
        		Pass `yes' to each $A_i$\;
		\tcp{If Q = \{EQ\} each sublearner will immediately complete}
        }{        
             	Get counterexample $\vec{x} = (x_1,\dots,x_k)$\;
            	\eIf{$\vec{x} \in \targ \backslash \prod S_i$}{ \tcp{Only happens if Q = \{EQ\}}
            		\For{i = 1 \dots k}{
               			\If{$x_i \not\in S_i$}{
            			Pass counterexample $x_i$ to $A_i$\;
            			Update $S_i$ to new query from $A_i$\;
            			}
              		}
             	}{
            		\For{i = 1 \dots k}{
            			Query $\vec{p}[i \leftarrow x_i] \in \targ$\;
            			\If{$\vec{p}[i \leftarrow x_i] \not\in \targ$ and $x_i \in S_i$}{
            				Pass counterexample $x_i$ to $A_i$\;
            				Update $S_i$ to new query from $A_i$\;
            			}
            		}
            	
            	}
	}
}
Each $A_i$ returns some $c_i$\;
Return $\prod c_i$\;
\caption{Algorithm for learning from Equivalence (or Subset)  Queries, Membership Queries, and One Positive Example}
\end{algorithm}

Finally, learning from only membership queries and one positive example if fairly easy. 

\begin{proposition}
\label{memandpos}
If $Q = \{\memQ\}$ and a single positive example $\vec{p} \in \targ$ is given, then $\targ$ is learnable in $k \cdot \sum \memCi{i}(c^*_i)$ membership queries. 
\end{proposition}
\begin{proof}
The algorithm learns by simulating each $A_i$ in sequence, moving on to $A_{i+1}$ once $A_i$ returns a hypothesis $c_i$. 
For any membership query $M_i$ made by $A_i$, $M_i \in c^*_i$ if and only if $\vec{p}[i \leftarrow M_i]\in \targ$ by Observation \ref{posobs}. 
Therefore the algorithm is successfully able to simulate the oracle for each $A_i$, yielding a correct hypothesis $c_i$. 
\end{proof}

\section{Learning From Only Membership Queries}
We have seen that learning with membership queries can be made significantly easier if a single positive example is given. 
In this section we describe a learning algorithm using membership queries when no positive example is given. 
This algorithm makes $O(max_i \{ \memCi{i}(c^*_i) \}^k)$ queries, matching the lower bound given in a previous section. 

For this algorithm to work, we need to assume that $\emptyset \not\in C_i$ for all $i$.
If not, there is no way to distinguish between an empty and non-empty concept. 
For example consider the classes $C_1 = \{ \{1\}, \emptyset \}$ and $C_2 = \{ \{j \} \mid j \in \mathbb{N} \}$. 
It is easy to know when we have learned the correct class in $C_1$ or in $C_2$ using membership queries. 
However, learning from their cross-product is impossible. 
For any finite number of membership queries, there is no way to distinguish between the sets $\emptyset$ and $\{(1,j)\}$ for some $j$ that has yet to be queried.

The main idea behind this algorithm is that learning from membership queries is easy once a single positive example is found. 
So the algorithm runs until a positive example is found from each concept or until all concepts are learned. 
If a positive example is found, the learner can then run the simple algorithm from Proposition \ref{memandpos} for learning from membership queries and a single positive example.

\begin{proposition}
Algorithm \ref{memonlyalg} will terminate after making $O(max_i \{ \memCi{i}(c^*_i) \}^k)$ queries.
\end{proposition}
\begin{proof}
The algorithm works by constructing sets $S_i$ of elements and querying all possible elements of $\prod S_i$. 
We will get our bound of $O(max_i \{ \memCi{i}(c^*_i) \}^k)$ by showing the algorithm will find a positive example once $| S_i | > max_i \{ \memCi{i}(c^*_i) \}$ for all $i$. 
Since the algorithm queries all possible elements of $\prod S_i$, it is sufficient to prove that $S_i$ will contain an element of $c^*_i$ once $|S_i| > \memCi{i}(c^*_i)$.
We will now show this is true for each $i$.

Assume that the sublearner $A_i$ eventually terminates with the correct answer $c^*_i$. 
Let $\vec{q_i} := (q_1^i, q_2^i, \dots) \in X_i^*$ be the elements whose membership $A_i$ would query assuming it only received negative answers from an oracle. 
If $\vec{q_i}$ is finite, then there is some set $N_i \in C_i$ that $A_i$ outputs after querying all elements in $\vec{q_i}$ (and receiving negative answers). 
We will consider the cases when $c^*_i = N_i$ and $c^*_i \ne N_i$

Assume $c^*_i = N_i$: Then by our assumption that $\targ \ne \emptyset$, $N_i$ contains some element $n_i$. 
Note that although sampling elements from a set might be expensive in general, this is only done for $N_i$ and can therefore be hard-coded into the learning algorithm. 
The algorithm will start with $S_i := \{n_i\}$, so $S_i$ contains an element of $c^*_i$ at the start of the algorithm. 

Assume $c^*_i \ne N_i$: By our assumption that $A_i$ eventually terminates, $A_i$ must eventually query some $q_j^i \in c^*_i$ (Otherwise, $A_i$ would only receive negative answers and would output $N_i$). 
So after $j$ steps, $S_i$ contains some element of $c^*_i$.
Since $j < \memCi{i}(c^*_i)$, we have that $S_i$ contains a positive example once $| S_i | > \memCi{i}(c^*_i)$, completing the proof.  
\end{proof}

\begin{algorithm}[H]
\label{memonlyalg}
\SetAlgoLined
\For{$i = 1 \dots k$}{
	\eIf{$N_i$ and $n_i$ exist}{
		Set $S_i := \{n_i\}$\;
	}{
		Set $S_i := \{ \}$\; 
	}
}
Set $j := 0$\;
\While{Some $A_i$ has not terminated}{
	\For{i = 1, \dots, k}{
		\If{$A_i$ has not terminated}{
			Get query $q_j^i$ from $A_i$ \;
			Pass answer `no' to $A_i$\; 
			Set $S_i := S_i \cup \{q_j^i\}$\;
		}
		%
	}
  \For{$\vec{x} \in \prod S_i$}{
  	Query $\vec{x} \in \targ$\;
	\If{$\vec{x} \in \targ$}{
		Run Proposition \ref{memandpos} algorithm using $\vec{x}$ as a positive example\;	
	}
  }
  $j := j+1$\;
}
\caption{Algorithm for Learning from Membership Queries Only}
\end{algorithm}

\section{Learning from Equivalence or Subset Queries is Hard}
The previous section showed that learning cross products of membership queries requires at most $O(max_i \{ \memCi{i}(c_i) \}^k)$ membership queries. 
A natural next question is whether this can be done for equivalence and subset queries. 
In this section, we answer that question in the negative. 
We will construct a class $\eqhard$ that can be learned from $n$ equivalence or subset queries but which requires at least $k^n$ queries to learn $\eqhard^k$.  

We define $\eqhard$ to be the set $\{ \ntreef(s) \mid s \in \mathbb{N}^* \}$, where $\ntreef(s)$ is defined as follows:

\[\ntreef(\lambda) := \{\lambda\} \times \mathbb{N}\]
\[\ntreef(s) := (\{s\} \times \mathbb{N}) \cup \ntree(s)\]
\[\ntree(sa) := (\{s\} \times (\mathbb{N} \backslash \{a\})) \cup \ntree(s)\]\\

For example, $\ntreef(12) = (\{12\} \times \mathbb{N}) \cup (\{1\}\times(\mathbb{N} \backslash \{2\}))\cup(\{\lambda\} \times (\mathbb{N} \backslash \{1\}))$.

An important part of this construction is that for any two strings $s,s' \in \mathbb{N}$, we have that $\ntreef(s) \subseteq \ntreef(s')$ if and only if $s = s'$. 
This implies that a subset query will return true if and only if the true concept has been found. 
Moreover, an adversarial oracle can always give a negative example for an equivalence query, meaning that oracle can give the same counterexample if a subset query were posed. 
So we will show that $\eqhard$ is learnable from equivalence queries, implying that it is learnable from subset queries. 

We will prove a lower-bound on learning $\eqhard^k$ from subset queries from an adversarial oracle. 
This will imply that $\eqhard^k$ is hard to learn from equivalence queries, since an adversarial equivalence query oracle can give the exact same answers and counterexamples as a subset query oracle.

\begin{proposition}
There exist algorithms for learning from equivalence queries or subset queries such that any concept $\ntreef(s) \in \eqhard$ can be learned from $|s|$ queries. 
\end{proposition}
\begin{proof}\todo{flesh out this proof?}
(sketch) Algorithm \ref{ntree} shows the learning algorithm for equivalence queries, and Figure \ref{singleeqhard} show the decision tree.  
When learning $\ntreef(s)$ for any $s \in \mathbb{N}^*$, the algorithm will construct $s$ by learning at least one new element of $s$ per query. 
Each new query to the oracle is constructed from a string that is a substring of $s$
If a positive counterexample is given, this can only yield a longer substring of $s$.
\end{proof}

\begin{algorithm}[H]
\label{ntree}
\SetAlgoLined
\KwResult{Learns $\eqhard$}
Set $s = \lambda$\;
\While{True}{
	Query $\ntreef(s)$ to Oracle
	\If{Oracle returns `yes'}{
		\Return $\ntreef(s)$
	}
	\If{Oracle returns $(s',m) \in c^* \backslash \ntreef(s)$ }
	{
		Set $s = s'$\;
	}
	\If{Oracle returns $(s,m) \in \ntreef(s) \backslash c^*$}{
		Set $s = sm$\;
	}
} 
\caption{Learning $\eqhard$ from equivalence queries.}
\end{algorithm}

\tikzset{
  treenode/.style = {align=center, inner sep=0pt, text centered,
    font=\sffamily},
    cnode/.style = {treenode, circle, black, font=\sffamily\bfseries, draw=black, text width=2.2em},
    elip/.style = {treenode,  draw=none, black, font=\sffamily\bfseries,  text width=2em},
    edge_style/.style={draw=black}
}

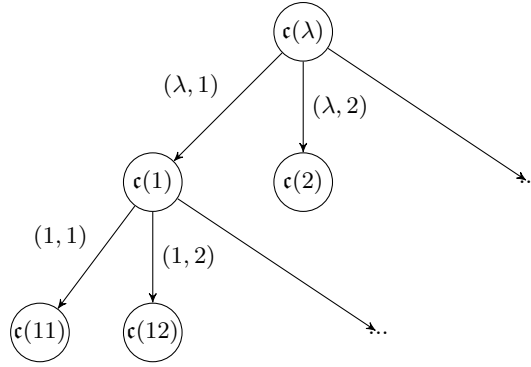
\begin{figure}
\centering
\begin{tikzpicture}[->,>=stealth',level/.style={sibling distance = 5cm/#1,
  level distance = 1.5cm}] 
\node[cnode] at (0, 0)   (lam) {$\ntreef(\lambda)$};
\node[cnode] at (-2, -2)   (1) {$\ntreef(1)$};
\node[cnode] at (0, -2)   (2) {$\ntreef(2)$};
\node[elip] at (3, -2)   (elp1) {$...$};
\node[cnode] at (-3.5, -4)   (11) {$\ntreef(11)$};
\node[cnode] at (-2, -4)   (12) {$\ntreef(12)$};
\node[elip] at (1, -4)   (elp2) {$...$};

 \draw[edge_style]  (lam) edge node[above left]{$(\lambda,1)$} (1);
 \draw[edge_style]  (lam) edge node[right]{$(\lambda,2)$} (2);
 \draw[edge_style]  (lam) edge (elp1);
 \draw[edge_style]  (1) edge node[above left]{$(1,1)$} (11);
 \draw[edge_style]  (1) edge node[right]{$(1,2)$} (12);
 \draw[edge_style]  (1) edge (elp2);

\end{tikzpicture}
\caption{A tree representing Algorithm \ref{ntree}. Nodes are labelled with the queries made at each step, and edges are labelled with the counterexample given by the oracle.}
\label{singleeqhard}
\end{figure}

\subsection{Showing $\eqhard^k$ is Hard to Learn}

It is easy to learn $\eqhard$, since each new counterexample gives one more element in the target string $s$. 
When learning a concept, $\prod \ntreef(s_i)$, it is not clear which dimension a given counterexample applies to. 
Specifically, a given counterexample $\vec{x}$ could have the property that $\vec{x}[i] \in \ntreef(s_i)$ for all $i \ne j$, but the learner cannot infer the value of this $j$. 
It must then proceed considering all possible values of $j$, requiring exponentially more queries for longer $s_i$. \todo{is this clear?} 
This subsection will formalize this notion to prove an exponential lower bound on learning $\eqhard^k$. 
First, we need a couple definitions.

A concept $\prod \ntreef(s_i)$ is \emph{justifiable} if one of the following holds:
\begin{itemize}
\item For all $i$, $s_i = \lambda$
\item There is an $i$ and an $a \in \mathbb{N}$ and $w \in \mathbb{N}^*$ such that $s_i = wa$, and the $k$-ary cross-product $\ntreef(s_1) \times \dots \times \ntreef(w) \times \dots \times \ntreef(s_k)$ was justifiably queried to the oracle and received a counterexample $\vec{x}$ such that $\vec{x}[i] = (w, a)$. 
\end{itemize}

A concept is \emph{justifiably queried} if it was queried to the oracle when it was justifiable. 
\newline

For any strings $s,s' \in \mathbb{N}^*$, we write $s \le s'$ if $s$ is a substring of $s'$, and we write $s < s'$ if $s \le s'$ and $s \ne s'$.
We say that the \emph{sum of string lengths} of a concept $\prod \ntreef(s_i)$ is of size $r$ if $\sum |s_i| = r$

Proving that learning is hard in the worst-case can be thought of as a game between learner and oracle. \todo{is this clear?}
The oracle can answer queries without first fixing the target concept. 
It will answer queries so that for any $n$, after less than $k^n$ queries, there is a concept consistent with all given oracle answers that the learning algorithm will not have guessed. 
The specific behavior of the oracle is defined as follows:

\begin{itemize}
\item It will always answer the same query with the same counterexample.  
\item Given any query $\prod \ntreef(s_i) \subseteq c^*$, the oracle will return a counterexample $\vec{x}$ such that for all $i$, $\vec{x}[i] = (s_i, a_i)$, and $a_i$ has not been in any query or counterexample yet seen.
\item The oracle never returns `yes' on any query. 
\end{itemize}

The remainder of this section assumes that queries are answered by the above oracle.
An example of answers by the above oracle and the justifiable queries it yields is given below.

\begin{example}
\label{eqhardex}
Consider the following example when $k = 2$. 
First, the learner queries $(\ntreef(\lambda), \ntreef(\lambda))$ to the oracle and receives a counter-example $((\lambda, 1), (\lambda, 2))$. 
The justifiable concepts are now $(\ntreef(1), \ntreef(\lambda))$ and $(\ntreef(\lambda), \ntreef(2))$. 
The learner queries $(\ntreef(1), \ntreef(\lambda))$ and receives counterexample  $((1, 3), (\lambda, 4))$. 
 The learner queries $(\ntreef(\lambda), \ntreef(2))$ and receives counterexample $((\lambda, 5), (2, 6))$.
The justifiable concepts are now  $(\ntreef(1), \ntreef(4))$, $(\ntreef(1 \cdot 3), \ntreef(\lambda))$, $(\ntreef(5), \ntreef(2))$ and $(\ntreef(\lambda), \ntreef(2 \cdot 6))$.  
At this point, these are the only possible solutions whose sum of string lengths is $2$.
The graph of justifiable queries is given in Figure \ref{eqhardtree}.
\end{example}

\tikzset{
  treenode/.style = {align=center, inner sep=0pt, text centered,
    font=\sffamily},
    cnode/.style = {treenode, rectangle, black, font=\sffamily\bfseries, draw=black, text width=9em, text height=1.5em},
    elip/.style = {treenode,  draw=none, black, font=\sffamily\bfseries,  text width=2em},
    edge_style/.style={draw=black}
}

\begin{figure}
\begin{tikzpicture}[->,>=stealth',level/.style={sibling distance = 5cm/#1,
  level distance = 1.5cm}] 
\node[cnode] at (-.5, 0)   (0) {JQ: $(\ntreef(\lambda),\ntreef(\lambda))$ \vspace{1mm} \\ CE: $((\lambda, 1),(\lambda, 2)) $ \vspace{2mm}   };

\node[cnode] at (-4, -2)   (10) {JQ: $(\ntreef(1),\ntreef(\lambda))$ \vspace{1mm} \\ CE: $((1, 3),(\lambda, 4))$ \vspace{2mm} };
\node[cnode] at (3.4, -2)   (11) {JQ: $(\ntreef(\lambda),\ntreef(2))$ \vspace{1mm} \\ CE: $((\lambda, 5),(2,6))$ \vspace{1mm} };

\node[cnode] at (-6, -4)   (20) {JQ: $(\ntreef(1\cdot 3),\ntreef(\lambda))$ \vspace{2mm} };
\node[cnode] at (-2.25, -4)   (21) {JQ: $(\ntreef(1),\ntreef(4))$ \vspace{2mm}};
\node[cnode] at (1.5, -4)   (22) {JQ: $(\ntreef(5),\ntreef(2))$ \vspace{2mm}};
\node[cnode] at (5.25, -4)   (23) {JQ: $(\ntreef(\lambda),\ntreef(2\cdot 6))$ \vspace{2mm}};

 \draw[edge_style]  (0) edge node[above left]{$1 \le s_1$} (10);
 \draw[edge_style]  (0) edge node[above right]{$2 \le s_2$} (11);
 
  \draw[edge_style]  (10) edge node[left]{$1,3 \le s_1$} (20);
 \draw[edge_style]  (10) edge node[right]{$4 \le s_2$} (21);
  \draw[edge_style]  (11) edge node[left]{$5 \le s_1$} (22);
 \draw[edge_style]  (11) edge node[right]{$2, 6 \le s_2$} (23);

\end{tikzpicture}
\caption{
The tree of justifiable queries used in Example \ref{eqhardex}. 
Each node lists the justifiable query (JQ) and counterexample (CE) given for that query. 
The edges below each node are labelled with the possible inferences about $s_1$ and $s_2$ that can be drawn from the counterexample.
}
\label{eqhardtree}
\end{figure}
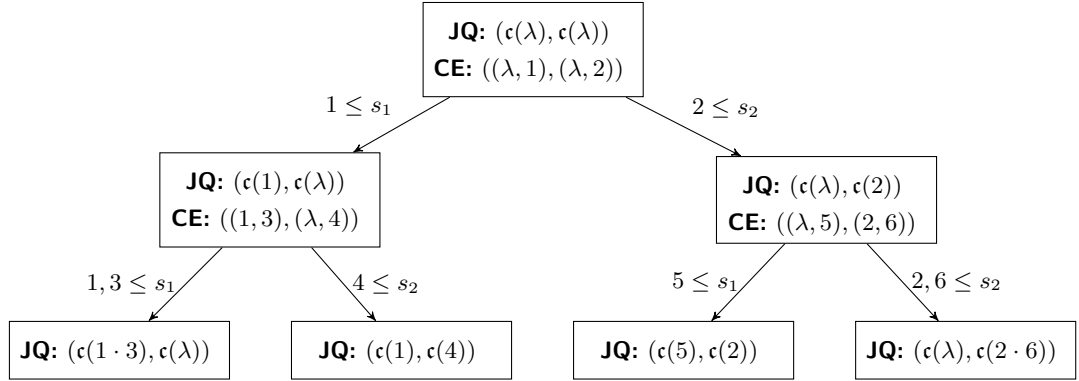
\todo{fix figure caption and box/text formatting}

The following simple proposition can be proven by induction on sum of string lengths.

\begin{proposition}
\label{subjust}
Let $\prod \ntreef(s_i)$ be a justifiable concept. 
Then for all $w_1$, $w_2$, \dots, $w_k$ where for all $i$, $w_i \le s_i$, $\prod \ntreef(w_i)$ has been queried to the oracle.
\end{proposition}

\begin{proposition}
\label{numjustconc}
If all justified concepts $\prod \ntreef(s_i)$ with sum of string lengths equal to $r$ have been queried, then there are $k^{r+1}$ justified queries whose sum of string lengths equals $r+1$
\end{proposition}
\begin{proof}
This proof follows by induction on $r$. 
When $r=0$, the concept $\prod \ntreef(\lambda)$ is justifiable.
For induction, assume that there are $k^r$ justifiable queries with sum of string lengths equal to $r$. 
By construction, the oracle will always chose counterexamples with as-yet unseen values in $\mathbb{N}$. 
So querying each concept $\prod \ntreef(s_i)$ will yield a counterexample $\vec{x}$ where for all $i$, $\vec{x}[i] = (s_i, a_i)$ for new $a_i$.
Then for all $i$, this query creates the justifiable concept $\prod \ntreef(s'_j)$, where $s'_j = s_j$ for all $j \ne i$ and $s'_i = \ntreef(s_i \cdot a_i)$.
Thus there are $k^{r+1}$ justifiable concepts with sum of string lengths equal to $r+1$.
\end{proof}

We are finally ready to prove the main theorem of this section.

\begin{theorem}
Any algorithm learning $\eqhard^k$ from subset (or equivalence) queries requires at least $k^r$ queries to learn a concept $\prod \ntreef(s_i)$, whose sum of string lengths is $r$.
Equivalently, the algorithm takes $k^{\sum \genCi{i}}$ subset (or equivalence) queries.
\end{theorem}
\begin{proof}\todo{Should I go into more detail why the existence of this c' shows the algorithm hasn't learned c?}
Assume for contradiction that an algorithm can learn with less than $k^r$ queries and let this algorithm converge on some concept $c = \prod \ntreef(s_i)$ after less than $k^r$ queries.
Since less than $k^r$ queries were made to learn $c$, by Proposition \ref{numjustconc}, there must be some justifiable concept $c' = \prod \ntreef(s'_i)$ with sum of string lengths less than or equal to $r$ that has not yet been queried.
By Proposition \ref{subjust}, we can assume without loss of generality that for all $w_i \le s_i'$, $\prod \ntreef(w_i)$ has been queried to the oracle.
We will show that $c'$ is consistent with all given oracle answers, contradicting the claim that $c$ is the correct concept. 
Let $c_v := \prod \ntreef(v_i)$ be any concept queried to the oracle, and let $\vec{x}$ be the given counterexample.
If for all $i$, $v_i \le s'_i$, then by construction, there is a $j$ with $\vec{x}[j] = (v_j, a_j)$ such that $v_j \cdot a_j \le s'_j$, so $\vec{x}$ is a valid counterexample.
Otherwise, there is an $i$ such that $v_i \not\le s'_i$. 
So $(\{v_i\} \times \mathbb{N})  \cap \ntreef(s'_i) = \emptyset$, so $\vec{x}$ is a valid counterexample. 
Therefore, all counterexamples are consistent with $c'$ being correct concept, contradicting the claim that the learner has learned $c$. 
\end{proof}

\section{Disjoint Union}
This section discusses learning disjoint unions of concept classes. 
This is generally much easier than learning cross-products of classes, since counterexamples belong to a single dimension in the disjoint union. 
This problem uses the same notation as the cross-product case, but we denote the disjoint union of two sets as $A \cupdot B$ and the disjoint union of many sets as $\bigcupdot A_i$.  
We define the concept class of disjoint unions as $\disClass := \{ \bigcupdot c_i \mid c_i \in C_i  \}$. 

The algorithm for learning from membership queries is very easy and won't be stated here. 
Algorithm \ref{disjalg} shows the learning procedure for when $Q \in \{ \{\subQ\}, \{\supQ\}, \{\eqQ\}\}$.
The correctness of this algorithm follows from the following simple facts.
Assume we have sets $S_1,\dots,S_k$ and $T_1,\dots,T_k$.
Then $\bigcupdot S_i \subseteq \bigcupdot T_i$ if and only $S_i \subseteq T_i$ for all $i$.
Likewise $\bigcupdot S_i = \bigcupdot T_i$ if and only if $S_i = T_i$ for all $i$.

We can summarize these results in the following proposition.

\begin{proposition}
Take any $Q \in \{ \{\eqQ\}, \{\subQ\}, \{\supQ\}, \{\memQ\}\}$ and assume each concept class $C_i$ is learnable from $\genC{i}$ many queries.  \todo{should I say $\genC{i}(c_i)$ instead?}
Then there exists an algorithm that can learn the disjoint union of concept classes $\bigcupdot c_i$ in $\sum \genC{i}$ many queries.  
\end{proposition}

\begin{algorithm}[H]
\label{disjalg}
\SetAlgoLined
\KwResult{Learning Disjoint Unions}
\For{$i = 1 \dots k$}{
	Set $S_i$ to initial query from $A_i$
}
\While{Some $A_i$ has not terminated}{
  Query $\bigcupdot S_i$ to oracle\;
  \eIf{Oracle returns `yes' }{
	Pass 'yes' to each $A_i$\;
   	Get updated $S_i$ from each $A_i$\; 
   }{
 	Get counterexample $x_i \in X_i$ for some $i$\;
	Pass $x_i$ as counterexample to $A_i$\;
	Get updated $S_i$ from each $A_i$\; 
	}
}
\Return{$\bigcupdot S_i$}\;
\caption{Learning Disjoint Unions}
\end{algorithm}

\section{Efficient PAC-Learning}
This sections discusses the problem of PAC-learning the cross products of concept classes. \todo{do we want to assume familiarity with VC dimension?}

Previously, van Der Vaart and Weller \cite{van2009note} have shown the following bound on the VC-dimension of cross-products of sets:
\[\VC(\prod C_i) \le a_1 log(k a_2) \sum \VC(C_i)\] 

Here $a_1$ and $a_2$ are constants with $a_1 \approx 2.28$ and $a_2 \approx 3.92$.
As always, $k$ is the number of concept classes included in the cross-product.

The VC-dimension gives a bound on the number of labelled examples needed to PAC-learn a concept, but says nothing of the computational complexity of the learning process. 
This complexity mostly comes from the problem of finding a concept in a concept class that is consistent with a set of labelled examples. 
We will show that the complexity of learning cross-products of concept classes is a polynomial function of the complexity of learning from each individual concept class.

First, we will describe some necessary background information on PAC-learning.

\subsection{PAC-learning Background}

\begin{definition}
Let $C$ be a concept class over a space $X$. 
We say that $C$ is efficiently PAC-learnable if there exists an algorithm $A$ with the following property:
For every distribution $\dist$ on $X$, every $c \in C$, and every $\epsilon, \delta \in (0, 1)$,
if algorithm $A$ is given access to $EX(c,\dist)$ then with probability $1 - \delta$, $A$ will return a $c' \in h$ such that $error(c') \le \epsilon$. 
$A$ must run in time polynomial in $1/\epsilon$, $1/\delta$, and $size(c)$. 
\end{definition}

We will refer to $\epsilon$ as the `accuracy' parameter and $\delta$ as the `confidence' parameter.
The value of $error(c)$ is the probability that for an $x$ sampled from $\dist$ that $c(x) \ne c^*(x)$. 
 PAC-learners have a \emph{sample complexity} function $m_C(\epsilon, \delta): (0,1)^2 \rightarrow \mathbb{N}$. 
The sample complexity is the number of samples an algorithm must see in order to probably approximately learn a concept with parameters $\epsilon$ and $\delta$.

Given a set $S$ of labelled examples in $X$, we will use $A(S)$ to denote the the concept class the algorithm $A$ returns after seeing set $S$ \todo{check whether we should use S as an input or as something given in alg}.

A learner $A$ is an \emph{empirical risk minimizer} if $A(C)$ returns a $c \in C$ that minimizes the number of misclassified examples (i.e., it minimizes $|\{(x,b) \in S \mid c^*(x) \ne b\}|$). 

Empirical risk minimizers are closely related to VC dimension and PAC-learnability as shown in the following theorem (Theorem 6.7 from \cite{shalev2014understanding})

\begin{theorem}
\label{fundpac}
If the concept class $C$ has VC dimension $d$, then there is a constant, $b$, such that applying an Empirical Risk Minimizer $A$ to $m_C(\epsilon, \delta)$ samples will PAC-learn in $C$, where
\[m_C(\epsilon, \delta) \le b \frac{d \cdot log(1/\epsilon) + log(1/\delta)}{\epsilon} \]
\end{theorem}

Finally, we will discuss the \emph{growth function}. 
The growth function describes how many distinct assignments a concept class can make to a given set of elements.  
More formally, for a concept class $C$ and $m \in \mathbb{N}$, the growth function $G_C(m)$ is defined by:
\[ G_C(m) = \max_{x_1,x_2,\dots,x_m} \abs[\Big]{\{ (c(x_1), c(x_2), \dots, c(x_m)) \mid c \in C\} }\]

Each $x_i$ in the above equation is taken over all possible elements of $X_i$.
The VC-dimension of a class $C$ is the largest number $d$ such that $G_C(d) = 2^d$. 

We will use the following bound, a corollary of the Perles-Sauer-Shelah Lemma, to bound the runtime of learning cross-products \cite{shalev2014understanding}. 

\begin{lemma}
\label{pss}
For any concept class $C$ with VC-dimension $d$ and $m > d+1$:
\[ G_C(m) \le (em/d)^d\]
\end{lemma}

\subsection{PAC-Learning Cross-Products}

We now have enough background to describe the strategy for PAC-learning cross-products.
We will just describe learning the cross-product of two concepts. \todo{do we want to describe the general strategy?} 
As above, assume concept classes $C_1$ and $C_2$ and PAC-learners $A_1$ and $A_2$ are given.
We define $A_i(\epsilon, \delta)$ as the runtime of the sublearner $A_i$ to PAC-learn with accuracy parameter $\epsilon$ and confidence parameter $\delta$. 

Assume that $C_1$ and $C_2$ have VC-dimension $d_1$ and $d_2$, respectively. 
We can use the bound from van Der Vaart and Weller to get an upper bound $d$ on the VC-dimension of their cross-product.
Assume the algorithm is given an $\epsilon$ and $\delta$ and there is a fixed target concept $c^* = c_1^* \times c_2^*$.
Theorem \ref{fundpac} gives a bound on the sample complexity $m_{C_1 \times C_2}(\epsilon, \delta)$. 
The algorithm will take a sample of labelled examples of size $m_{C_1 \times C_2}(\epsilon, \delta)$. 
Our goal is to construct an Empirical Risk Minimizer for $C_1 \times C_2$. 
In our case, $c_1 \in C_1$ and $c_2 \in C_2$.
Therefore,  for any sample $S$, an Empirical Risk Minimizer will yield a concept in $C_1 \times C_2$ that is consistent with $S$. 
This algorithm is show in Algorithm \ref{recpaclearn}.

So let $S$ be any such sample the algorithm takes. 
This set can easily be split into positive examples $S^+$ and negative examples $S^-$, both in $X_1 \times X_2$.
The algorithm works by maintaining sets labeled samples $L_1$ and $L_2$ for each dimension. 
For any $(x_1, x_2) \in S^+$, it holds that $x_1 \in c^*_1$ and $x_2 \in c^*_2$ so $(x_1, \top)$ and $(x_2, \top)$ are added to $L_1$ and $L_2$ respectively. 
For any $(x_1, x_2) \in S^-$, we know that $x_1 \not\in c^*_1$ or $x_2 \not\in c^*_2$ (or both), but it is not clear which is true. 
However, since the goal is only to create an Empirical Risk Minimizer, it is enough to find any concepts $C_1$ and $C_2$ that are consistent with these samples. 
In other words, we need to find a $c_1 \in C_1$ and a $c_2 \in C_2$ such that for every $(x_1, x_2) \in S^+$, $x_1 \in c_1$ and $x_2 \in c_2$ and for all $(x_1, x_2) \in S^-$, either $x_1 \not\in c^*_1$ or $x_2 \not\in c^*_2$.
One idea would be to try out all possible assignments to elements in $S^-$ and check if any such assignment fits any possible concepts. \todo{is it clear what "all possible assignments" are?}
This, however, would be exponential in $|S^-|$. 

Bounding the size of the growth function can narrow this search. 
Specifically, let $S_1^- := \{ x \mid \exists y, (x,y) \in S^-\}$, let $m = |S_1^-|$ and order the elements of $S_1^-$ by $x_1, x_2, \dots, x_m$. 
By the definition of the growth function and Lemma \ref{pss}:
\[ |\{ (c(x_1), c(x_2), \dots, c(x_m)) \mid c \in C_1\}| \le G_{C_1}(m) \le (em/d)^d \]
In other words, there are less than $(em/d)^d$ assignments of truth values to elements of $S_1^-$ that are consistent with some concept in $C_1$.
If the algorithm can check every $c_1 \in C_1$ consistent with $S^+$ and $S^-_1$, it can then call $A_2$ to see if there is any $c_2 \in C_2$ such that $(c_1 \times c_2)$ assigns true to every element in $S^+$ and false to every element in $S^-$. \todo{be clear about difference between a set of pairs of labeled examples (such as $S^-$) and one side of that set (such as $S^-_1$)} 

Finding these consistent elements of $C_1$ is made easier by the fact that we can check whether partial assignments to $S^-_1$ are consistent with any concept in $C_1$. 
As mentioned above, it starts by creating the sets $L_1$ and $L_2$ containing all samples in the first and second dimension of $S^+$, respectively.
It then iteratively adds labeled samples from $S^-$.
At each step, the algorithm chooses one element $(x_1, x_2) \in S^-$ at a time and checks which possible assignments to $x_1$ are consistent with $L_1$.
If $(x_1, \bot)$ is consistent, it adds $(x_1, \bot)$ to $L_1$ and calls $RecursiveFindSubconcepts$ on $L_1$ and $L_2$.
 \todo{just add a description saying a sample is consistent with a class when theres a concept that fits the sample}
If $(x_1, \top)$ is consistent with $C_1$, then the algorithm adds $(x_1, \top)$ to $L_1$ and $(x_2, \bot)$ to $L_2$ and calls $RecursiveFindSubconcepts$.
In either case, if an assignment is not consistent, no recursive call is made. 
We can summarize these results in the following theorem.


\begin{theorem}
Let concept classes $C_1$ and $C_2$ have VC-dimension $d_1$ and $d_2$, respectively.
There exists a PAC-learner for $C_1 \times C_2$ that can learn any concept using a sample of size $m = ((d_1 + d_2) \cdot log(1/\epsilon) + log(1/\delta))/\epsilon$.\todo{there are some constants in there}
The learner requires time $O(m^{d_1}(A_1(1/m, log(\delta)) +  A_2(1/m, log(\delta))))$.
\end{theorem}

\todo{testing}


\begin{algorithm}[H]
\label{recpaclearn}
\SetAlgoLined
\KwResult{Find Subconcepts Consistent with Sample}
\SetKwProg{FMain}{FindSubconcepts}{}{}
\SetKwProg{FRec}{RecursiveFindSubconcepts}{}{}

\KwIn{
$S^+$: Set of positive examples in $X_1 \times X_2$\\
$S^-$: Set of negative examples in $X_1 \times X_2$\\
$\delta:$ Confidence parameter in $(0,1)$
}
 \FMain{($S^+$, $S^-$, $\delta$)}{
 $\delta' := \delta / (|S^-|G_{C_1}(|S^-|) + G_{C_2}(|S^-|))$\;\todo{describe this parameter}
 $\epsilon' := 1 / |S|$\;
 \tcp{$L_1$: Labelled samples in $X_1$}
 $L_1 := \{(x_1, \top) \mid \exists y,  (x_1, y) \in S^+ \}$\;
  \tcp{$L_2$: Labelled samples in $X_2$}
 $L_2 := \{(x_2, \top) \mid \exists y,  (y, x_2) \in S^+ \}$\; 
 $U := S^-$\;
 \Return RecursiveFindSubconcepts($L_1$, $L_2$, $U$, $\epsilon'$, $\delta'$) \;
 }

\FRec{($L_1$, $L_2$, $U$, $\epsilon'$, $\delta'$):}{

\tcp{Run once all labels are assigned to the first dimension of $U$}
\tcp{Attempts to find concept in $C_2$ consistent with all labels given in $L_2$}
\If{$U = \emptyset$}{
	\eIf{$A_2(L_2, \epsilon', \delta')$}{
		\Return{$(A_1(L_1, \epsilon', \delta'), A_2(L_2, \epsilon', \delta'))$}\;
	}{$\bot$\;}	
}

Get $(x_1, x_2) \in U$\;
$U := U \backslash \{(x_1, x_2)\} $\;

 \tcp{Attempts to label $x_1$ as false in $c^*_1$}
\If{$A_1(L_1 \cup \{ (x_1, \bot) \}, \epsilon', \delta') \ne \bot$}{
	$c = RecursiveFindSubconcepts(L_1 \cup \{ (x_1, \bot) \}, L_2, U, \epsilon', \delta')$\;
	\If{$c \ne \bot$}{\Return c\;}
}

\tcp{Attempts to label $x_1$ as true in $c^*_1$}
\If{$A_1(L_1 \cup \{ (x_1, \top) \}, \epsilon', \delta') \ne \bot$}{
	$c = RecursiveFindSubconcepts(L_1 \cup \{ (x_1, \top) \}, L_2 \cup \{(x_2, \bot) \}, U, \epsilon', \delta')$\;
	\If{$c \ne \bot$}{\Return c \;}
}

}

\caption{Learning Disjoint Unions}
\end{algorithm}

\subsection{Efficient PAC-learning with Membership Queries}

Although polynomial, the complexity of PAC-learning cross-products from a $\pacQ$ oracle is fairly expensive. 
We will show that when a learner is allowed to make membership queries, PAC-learning cross-products becomes much more efficient. 
This is due to the previously shown technique, which uses membership queries and a single positive example to determine on which dimensions a negatively labelled example fails. 

In this case, assuming that $\emptyset \in \Boxtimes C_i$, we can ignore the assumption that a positive example is given. 
If no positive example appears in a large enough labeled sample, the the algorithm can pose $\emptyset$ as the hypothesis.\todo{show that this pac-learns?}


If $S$ does contain a positive example $\vec{p}$, then $S$ can be broken down into labeled samples for each dimension $i$. 
The algorithm initialize the sets of positive and negative examples to $S^+_i := \{ \vec{x}[i] \mid  (\vec{x}, \top) \in S\}$ and $S^-_i := \{ \}$, respectively.
For each $(\vec{x}, \bot) \in S$, a membership queries $\vec{p}[i \leftarrow \vec{x}[i]] \in \targ$. 
If so, $\vec{x}[i] $ is added to $S^+_i$.
Otherwise it is added to $S^-_i$.
This labelling is correct by Observation \ref{posobs}.
The set of labelled examples $S_i := (S_i^+ \times \{\top\}) \cup (S_i^- \times \{\bot\})$ is then passed to the sublearner $A_i$.
$A_i$ is run on $S_i$ with accuracy parameter $\epsilon' := \epsilon / k$ and confidence parameter $\delta' := \delta / k$ \todo{DoubleCHECK THIS}.

\begin{proposition}
The algorithm described above PAC-learns from the concept class $\Boxtimes C_i$ with accuracy $\epsilon$ and confidence $\delta$.
It makes $m_C(\epsilon, \delta)$ queries to $\pacQ$, $k \cdot m_C(\epsilon, \delta)$ membership queries, and has runtime $O(\sum A_i(\epsilon / k, \delta / k))$.
\end{proposition}

\section{Conclusion}
The final collection of query complexities for learning cross products is given in Figure \ref{complexitytable}.
All of the bounds are tight, except for the problem of learning with superset queries, membership queries, and one positive example.  
Additionally, learning disjoint unions from any $Q \in \{ \{\eqQ\}, \{\subQ\}, \{\supQ\}\}$ requires as many queries as is needed to learn each concept separately. 
Finally, we have shown that the computational complexity of PAC learning cross-products of concepts is a polynomial function of learning the individual concepts.

\begin{figure}
\begin{center}
\renewcommand{\arraystretch}{1.5}
\begin{tabular}{ |c|c|c|c| } 
\cline{2-4}
\multicolumn{1}{c|}{} & Only $Q$ & \multicolumn{2}{c|}{$Q$ with $\memQ$ and $\oneposQ$} \\
\cline{1-1}
\multicolumn{1}{|c|}{$Q \downarrow$} & \genC & \memC & \genC \\
\hline
\posQ & Not Possible  & Not Possible & Not Possible \\
\hline
\supQ & $\sum \supCi{i}$ & $0$ & $\sum \supCi{i}$\\
\hline
\memQ & $(max_i\{\memCi{i}\})^k$ & $\sum \memCi{i}$ & $\sum \memCi{i}$ \\
\hline
\subQ & $k^{\sum \subCi{i}}$ & $k \sum \subCi{i}$  & $\sum \subCi{i}$ \\
\hline
\eqQ  &$k^{\sum \eqCi{i}}$ &  $k \sum \eqCi{i}$ &  $\sum \eqCi{i}$\\
\hline
\end{tabular}
\renewcommand{\arraystretch}{1}
\end{center}
\caption{
Final collection of query complexities for learning cross products. 
The rows represents the set $Q$ of queries needed to learn each $C_i$.  
The columns determine whether the cross product is learned from queries in just $Q$ or $Q \cup \{\memQ, \oneposQ\}$. 
In the latter case, the column is separated to track the number of membership queries and queries in $Q$ that are needed.
The value $k$ denotes the number of dimensions (i.e., concept classes) included in the cross-product.
 }
 \label{complexitytable}
\end{figure}\todo{Is there a way to add PAC and disjunion results to this table, or make them apparent in this section?}

\bibliographystyle{unsrt}
\bibliography{bibliography}

\end{document}